\documentclass[final, 3p, times,review]{elsarticle}

\usepackage{lineno,hyperref}

\usepackage{amsmath}
\usepackage{booktabs}
\usepackage{graphicx}
\usepackage{multirow}

\usepackage{amssymb}
\usepackage{amsthm}
\usepackage{enumerate}
\usepackage{color}

\newtheorem{theorem}{Theorem}
\newtheorem{proposition}{Proposition}

\newtheorem{lemma}{Lemma}
\newtheorem{definition}{Definition}

\newcommand{\R}{\mathbb{R}}
\DeclareMathOperator{\dive}{div}
\DeclareMathOperator{\vol}{vol}
\newcommand{\norm}[1]{\left\lVert#1\right\rVert}

\begin{document}

\begin{frontmatter}

\title{VPNets: Volume-preserving neural networks for learning source-free dynamics}
\author[cas,ucas]{Aiqing Zhu}
\author[ustb]{Beibei Zhu}
\author[cas,ucas]{Jiawei Zhang}
\author[cas,ucas]{Yifa Tang}
\author[ustc,qlut]{Jian Liu\corref{cor}}
\cortext[cor]{Corresponding author: jliuphy@ustc.edu.cn (Jian Liu)}

\address[cas]{LSEC, ICMSEC, Academy of Mathematics and Systems Science, Chinese Academy of Sciences, Beijing 100190, China}
\address[ucas]{School of Mathematical Sciences, University of Chinese Academy of Sciences, Beijing 100049, China}
\address[ustb]{School of Mathematics and Physics, University of Science and Technology Beijing, Beijing 100083, China}
\address[ustc]{School of Nuclear Science and Technology, University of Science and Technology of China, Hefei, Anhui 230026, China}
\address[qlut]{Advanced Algorithm Joint Lab, Shandong Computer Science Center, Qilu University of Technology, Jinan, Shandong 250014, China}

\begin{abstract}
We propose volume-preserving networks (VPNets) for learning unknown source-free dynamical systems using trajectory data. We propose three modules and combine them to obtain two network architectures, coined R-VPNet and LA-VPNet. The distinct feature of the proposed models is that they are intrinsic volume-preserving. In addition, the corresponding approximation theorems are proved, which theoretically guarantee the expressivity of the proposed VPNets to learn source-free dynamics. The effectiveness, generalization ability and structure-preserving property of the VP-Nets are demonstrated by numerical experiments.
\end{abstract}

\end{frontmatter}


\section{Introduction}\label{sec:Introduction}
Data-driven discovery has received increasing attention in diverse scientific disciplines. There are extensive attempts to treat this problem using symbolic regression \cite{bongard2007automated,schmidt2009distilling}, Gaussian process \cite{raissi2017machine,kocijan2005dynamic} as well as Koopman theory \cite{brunton2017chaos}. Along with the rapid advancements of machine learning, neural networks were proposed to handle this problem and proven to be a valuable tool due to their remarkable abilities to learn and generalize from data. The pioneering efforts on using neural networks for discovery dated back to the early 1990s \cite{anderson1996comparison,gonzalez1998identification,rico1994continuous,rico1993continuous}, where they combined neural networks (NNs) and numerical integration to reconstruct the unknown governing function and hence depict the trajectories. Recently, this idea has been further explored and applied to more challenging tasks \cite{chen2018neural,kolter2019learning,qin2019data, raissi2018multistep}.

Recently, researchers empirically observed that encoding prior physical structures into the learning algorithm can enlarge the information content of the data, and yield trained models with good stability and generalization. For example, OnsagerNets \cite{yu2020onsagernet} embedded generalized Onsager principle into the learning model to retain physical structure including free energy, dissipation, conservative interaction and external force. GFINNs \cite{zhang2021gfinns} were proposed to obey the symmetric degeneracy conditions via orthogonal modules for the GENERIC formalism. A special structure for learning nonlinear operators was embedded in DeepONets \cite{lu2021learning}, where its performance was verified across diverse applications. For more extensive works on structure-preserving deep learning, we refer to \cite{celledoni2021structure}.

In particular, incorporating Hamiltonian equation or symplectic structure into neural networks has been widely studied and many satisfactory results have been obtained. Recent works \cite{bertalan2019learning, chen2020symplectic, greydanus2019hamiltonian, tong2021symplectic, wu2020structure, xiong2021nonseparable, zhong2020symplectic} mainly focus on approximating Hamiltonian vector field from phase space data by means of using numerical integration to reconstruct symplectic map. Most of the aforementioned approaches rely on the vector field and may introduce additional numerical errors during training \cite{du2021discovery,keller2021discovery,zhu2020inverse} and predicting processes. Regarding this issue, GFNN \cite{chen2021data} was proposed to learn generating function in order to reconstruct symplectic map. Theoretical and experimental results show that the global error of GFNN increases linearly. SympNets \cite{jin2020sympnets} stacked up triangular maps to construct new intrinsic symplectic networks, where rigorous approximation theorems
were built.

More generally than Hamiltonian systems, source-free systems are classical cases of dynamical systems with certain geometric structures and exist in many physical fields such as plasmas and incompressible fluids.  A remarkable property of source-free systems is that their latent flow map is volume-preserving. Specifically, the Hamiltonian system is source-free and the symplectic map is volume-preserving. Compared with Hamiltonian systems, researchers pay less attention to source-free dynamics. Due to the superiority of structure-preserving properties, our goal is to embed volume-preserving structure into neural networks. Learning dynamics plays an important role in various applications of machine learning such as robotic manipulation \cite{hersch2008dynamical}, autonomous driving \cite{levinson2011towards} and other motion planning tasks. Many studies have demonstrated the significance of encoding inductive biases based on physical laws into neural network architectures \cite{lake2017building,marcus2020next}. However, the question of which structure should be incorporated into the model still remains open. The volume-preserving neural networks that we construct in this paper investigate the less explored volume-preserving structure, and potentially open a new path of learning real world dynamics.

To begin with, we present some preliminary knowledge. A map $F:\R^D \rightarrow \R^D$ preserves volume if for every bounded open set $\Omega \subset \R^D$
\begin{equation}\label{eq:vp1}
\vol(F(\Omega)) = \vol(\Omega),
\end{equation}
where $\vol(\Omega) = \int_{\Omega} dy$. In particular, $F$ is volume-preserving if $F$ is bijective, continuously differentiable and $\det \frac{\partial F(x)}{\partial x} = 1$ due to the transformation formula for integrals.
Since the determinant of symplectic matrix is also one, volume-preservation is more general than symplectic structure. A continuous dynamical system can be written as
\begin{equation*}
\frac{d y}{dt} = f(y(t)), \quad y(0)=x,
\end{equation*}
where $y \in \R^D$. Let $\phi_t(x)$ be the phase flow with initial condition $y(0)=x$. If the vector field $f(y)$ is source-free, i.e., $\dive f(y) = \sum_i^D \frac{\partial f}{\partial y_i} = 0$, then $\phi_t(x)$ preserves volume in phase space \cite{hairer2006geometric}, viz., $\det \frac{\partial \phi_t(x)}{\partial x} = 1$. There have been many efforts focused on constructing volume-preserving networks in the literature. A volume-preserving approach \cite{macdonald2021volume} was proposed to lessen the vanishing (or exploding) gradient problem in deep learning. Locally symplectic neural networks \cite{bajars2021locally} were developed recently to learn volume-preserving dynamics. However, these models lack the approximation guarantees, and no approximation theorem has been proven. In addition, NICE \cite{dinh2015nice} were proposed for unsupervised generative modeling. Reversible residual networks (RevNets) \cite{gomez2017the} were proposed to avoid storing intermediate activation during backpropagation. It is proved that NICE or RevNets are able to approximate every volume-preserving map in our previous work \cite{zhu2022approximation}. However, these two models are not designed for learning dynamics. In fact, one of our proposed architectures is an extension of NICE using dimension-splitting mechanisms investigated in \cite{zhu2022approximation}.

In this paper, we propose new intrinsic volume-preserving neural networks to learn source-free dynamics by directly observing the system's states. The reconstructed and predicted dynamics can automatically satisfy the volume-preserving property. We also prove the approximation theorems, demonstrating that the proposed VPNets are sufficiently expressive to learn any source-free dynamic, respectively. Over all, the main contributions of our work can be summarized as below:
\begin{itemize}
    \item We develop intrinsic volume-preserving models using neural networks.
    \item We prove that the proposed models are capable of approximating arbitrary volume-preserving flow maps.
    \item Numerically, the proposed models can learn and predict source-free dynamics by directly observing the system's states, even if such observation is partial and the sample data is sparse.
\end{itemize}

The rest of this paper is organized as follows. Section \ref{sec: The VPNets architecture} introduces the detailed procedure of the learning algorithm and the construction of the VPNets. The approximation theorem of proposed VPNets is presented in Section \ref{sec: Approximation results of VPNets}. Section \ref{sec: Numerical results} provides several experimental results for source-free systems. In Section \ref{sec: Conclusions and future works}, we give a brief summary and discuss future directions.

\section{Learning method and the VPNet architectures}\label{sec: The VPNets architecture}
Consider a continuous dynamical system
\begin{equation}\label{eq:ODE}
\frac{d y}{dt} = f(y(t)), \quad y(0)=x,
\end{equation}
where $y \in \R^D$. Let $\phi_t(x)$ be the exact solution of (\ref{eq:ODE}) with initial condition $y(0)=x$. In this paper, we aim to learn the phase flow $\phi_T$ of a unknown dynamical system from data $ \mathcal{T}= \{(x_i, y_i)\}_{i=1}^I$, where $y_i = \phi_T(x_i)$, $x_i,y_i \in \R^D$. Typically, training data are the states at equidistant time steps of one or more trajectories, i.e., $x_1, \cdots, x_{I+1}$ where $x_i = \phi_{iT}(x_1)$. These data points can be grouped in pairs and written as $\mathcal{T}= \{(x_i, y_i)\}_{i=1}^I$ where $y_i =x_{i+1}= \phi_T(x_i)$. Network is trained by minimizing the mean-squared-error loss
\begin{equation}\label{eq:mse}
\text{MSE} =  \frac{1}{D \cdot I}\sum_{i =1}^I \norm{\psi_{net}(x_i) - y_i},
\end{equation}
where $\psi_{net}$ is neural networks with trainable parameters. This task appears in many contexts (see Section \ref{sec:Introduction}). Herein, we consider a very specific one, i.e., $\phi_T$ is the phase flow of a source-free system. Same as symplecticity requirement for learning Hamiltonian system, we should carefully construct networks to ensure that the learning model $\psi_{net}$ has volume-preserving structure (\ref{eq:vp1}) since the flow of source-free system preserves volume.

In this paper, as intrinsic volume-preserving structure is embedded into the network $\psi_{net}$, we name the network as volume-preserving neural networks (VPNets). A VPNet is highly flexible via composing the following three alternative modules that we present below. We will introduce two kinds of VPNets, from the perspective of both approximation and simulation.

\begin{table}[htbp]
    \centering
    \begin{tabular}{p{70pt}|p{370pt}}
        \toprule
        $x[i]$ & The $i$-th component (row) of vector (matrix) $x$.\\
        \midrule
        $x[:][j]$ & The $j$-th column of matrix $x$.\\
        \midrule
        $x[i_1:i_2]$ & $(x[i_1],\cdots, x[i_2-1])^\top$ if $x$ is a column vector or $(x[i_1],\cdots, x[i_2-1])$ if $x$ is a row vector, i.e., components from $i_1$ inclusive to $i_2$ exclusive. \\
        \midrule
        $x[\ :i ]$ and $x[i:\ ]$& $x[1:i]$ and  $x[i:D+1]$ for $x \in \R^D$, respectively. \\
        \midrule
        $\overline{x[i_1:i_2]}$ & $(x[\ :i_1]^\top, x[i_2:\ ]^\top)^\top$ if $x$ is a column vector or $(x[\ :i_1], x[i_2:\ ])$ if $x$ is a row vector, i.e., components in the vector $x$ excluding $x[i_1:i_2]$.\\
        \bottomrule
    \end{tabular}
    \caption{Range indexing notations in this paper.}
    \label{tab:notations}
\end{table}

For convenience, range indexing notation, the same kind for Pytorch tensors, is employed throughout this paper. With the help of NumPy and other Python scientific libraries, we can apply the range indexing notation for each dimension of the tensor. Details are present in Table \ref{tab:notations}.

\subsection{Residual modules}
To begin with, we propose residual modules which partition input $x$ and produce output according to the following rule:
\begin{equation*}
\mathcal{R}^{i:j}(x)= \begin{pmatrix}x[\ :i]\\ x[i:j] + \hat{\sigma}_{\theta}(\overline{x[i:j]})\\ x[j:\ ] \end{pmatrix}.
 \end{equation*}
where $\hat{\sigma}_{\theta}$ specifies a neural network with the trainable parameters $\theta$. We set $\hat{\sigma}_{\theta}$ to be a fully connected network with one hidden layer,
\begin{equation*}
\hat{\sigma}_{\theta}(\overline{x[i:j]}) =a \sigma(K\overline{x[i:j]}+b).
\end{equation*}
Here, $\theta = (K,b,a)$ and $K \in \R^{w \times (D-j+i)}, b\in \R^{w}, a\in \R^{(j-i) \times w}$ with width $w$ are trainable parameters and $\sigma: \R \rightarrow \R$ is the activation function applied element-wise to a vector. Popular examples for activation function include the rectified linear unit (ReLU) $\text{ReLU}(z) = \max(0, z)$,
the sigmoid $\text{Sig}(z) = 1/(1 + e^{-z})$ and $\text{tanh}(z)$. The residual module is named due to the residual representation.

This module is inspired by NICE \cite{dinh2015nice} and we use special dimension-splitting mechanism. In the following, we denote the set of the residual modules as
\begin{equation*}
\mathcal{M}_R = \{u\ |\ u \text{ is a residual module}\},
\end{equation*}
and define the composition of residual modules as residual volume-preserving networks (R-VPNets):
\begin{definition}
Consider $u_n \in \mathcal{M}_R$ for $n=1,2,\cdots,N$ and take
\begin{equation*}
\psi_R= u_{N} \circ u_{N-1} \circ \cdots  \circ u_1,
\end{equation*}
where $N$ is the depth. We name $\psi$ as R-VPNet and denote the collection of R-VPNets as
\begin{equation*}
\Psi_R = \{ \psi\ |\ \psi \text{ is a R-VPNet} \}.
\end{equation*}
\end{definition}

\subsection{Linear modules and activation modules}
In addition, we propose LA-VPNets motivated by LA-SympNets. LA-VPNets do not change the approximation properties of the network and are also volume-preserving. These models are composed of linear modules and activation modules. The linear modules are linear transformations preserving the volume and play similar roles as the linear layers do in fully connected neural networks.
For $i = 1, \cdots ,D$ and $i<j\leq D+1$, we denote
\begin{equation}\label{eq:matrix}
L^{i:j} = \left\{ S^{i:j} \in \R^{D \times D} \ \Bigg|\ S^{i:j} = \begin{pmatrix} I_{(i-1) \times (i-1)} & 0_{(i-1) \times (j-i)} & 0_{(i-1) \times (D-j+1)} \\ U_{(j-i) \times (i-1)} & I_{(j-i) \times (j-i)} & V_{(j-i) \times (D-j+1)} \\ 0_{(D-j+1) \times (i-1)} & 0_{(D-j+1) \times (j-i)} & I_{(D-j+1) \times (D-j+1)} \end{pmatrix}\right\},
\end{equation}
where $I_{s \times s}$ is the $s$-by-$s$ identity matrix and $U \in \R^{(j-i)\times (i-1)}, V\in \R^{(j-i)\times (D-j+1)}$ are trainable parameters. The subscript indicates the shape of matrices. In order to strengthen the expressivity, the linear modules are compounded from several matrices of the form (\ref{eq:matrix}), more precisely,
\begin{equation*}
\mathcal{L}(x) = \left(\prod_{m=1}^M S_m \right) x +b,
\end{equation*}
where $S_m \in \cup_{i=1}^{D} \cup_{j=i+1}^{D+1} L^{i:j}$  and trainable bias $b \in \R^{D}$. We denote the set of the linear modules as
\begin{equation*}
\mathcal{M}_L = \{v\ |\ v \text{ is a linear module}\}.
\end{equation*}
All linear modules are automatically volume-preserving since $\det S_m =1$ without any constraints. With the unconstrained parametrization, we can apply efficient unconstrained optimization algorithms of the deep learning framework. We also remark that $\mathcal{M}_L$ can approximate any linear volume-preserving maps and we will prove this remark in Section \ref{sec: Approximation results of VPNets}.

To substitute for the activation layer in fully connected neural networks, we build a simple nonlinear volume-preserving module. The architecture is designed as
\begin{equation*}
\mathcal{A}^{i:j}(x)= \begin{pmatrix}x[\ :i]\\ x[i:j] + a \sigma(\overline{x[i:j]})\\ x[j:\ ] \end{pmatrix},
\end{equation*}
where $a \in \R^{(j-i) \times(D-j+i)}$ are trainable parameters which are added to ensure approximation ability, and $\sigma: \R \rightarrow \R$ is the activation function applied element-wise to a vector. Similar to the linear modules, the set of the activation modules is denoted as
\begin{equation*}
\mathcal{M}_A = \{v \ |\ v \text{ is a activation module}\}.
\end{equation*}
We define the composition of the linear modules and the activation modules as LA-VPNets.
\begin{definition}
Consider $l_n \in \mathcal{M}_L$ for $n=1,2,\cdots,N+1$ and $a_n \in \mathcal{M}_A$ for $n=1,2,\cdots,N$, take
\begin{equation*}
\psi_{LA} = l_{N+1} \circ a_{N} \circ l_{N} \circ \cdots  \circ a_1\circ l_{1},
\end{equation*}
where $N$ is the depth. We name $\psi_{LA}$ as LA-VPNet and denote the collection of LA-VPNets as
\begin{equation*}
\Psi_{LA}= \{ \psi\ |\ \psi \text{ is a LA-VPNet} \}.
\end{equation*}
\end{definition}
It will be shown in Section \ref{sec: Approximation results of VPNets} that any source-free flow map can be approximated by LA-VPNets.

\section{Approximation results of VPNets}\label{sec: Approximation results of VPNets}
The attention in this section will be addressed to the approximation theorem. To begin with, we introduce some notations. Consider a differential equation
\begin{equation}\label{eq:ODE2}
\frac{d}{dt}y(t) = f(t, y(t)),\quad y(\tau)=x, \tau\geq 0,
\end{equation}
where $y(t) \in \R^D$, $f \in C^{1}([0,+\infty)\times \R^D)$. For a given time step $T\geq 0$, $y(\tau+T)$ could be regarded as a function of its initial condition $x$. We denote $\phi_{\tau,T,f}(x) := y(\tau+T)$, which is known as the time-$T$ flow map of the dynamical system (\ref{eq:ODE2}). We also write the collection of such flow maps as
\begin{equation*}
    \mathcal{F}(U) = \left\{\phi_{\tau, T, f}: U \rightarrow \R^D \ \big|\ \tau,T\geq0,\ f \in C^{1}([0,+\infty)\times \R^D) \right\}.
\end{equation*}
In particular, we denote the set of measure-preserving flow maps as
\begin{equation*}
    \mathcal{VF}(U) = \left\{\phi_{\tau, T, f}\in \mathcal{F}(U)\ \big|\ \dive_y f=0\right\}.
\end{equation*}
We will work with $C$ norm and denote the norm of map $F$ as
\begin{equation*}
\norm{F}_{U}=\max_{1 \leq d \leq D}\sup_{x\in U} \lvert F_d(x) \rvert.
\end{equation*}
Now, the approximation theorems are given as follows.
\begin{theorem}[Approximation theorem for R-VPNets]\label{thm:R}
Given a compact set $U\subset \R^{D}$ and a volume-preserving flow map $\phi \in \mathcal{VF}(U)$, for any $\epsilon>0$, there exists $\psi \in \Psi_{R}$ such that $\norm{\phi - \psi}_{U}<\epsilon$. Here, $\Psi_{R}$ is the set of R-VPNets.
\end{theorem}

\begin{theorem}[Approximation theorem for LA-VPNets]\label{thm:LA}
Given a compact set $U\subset \R^{D}$ and a volume-preserving flow map $\phi \in \mathcal{VF}(U)$, for any $\epsilon>0$, there exists $\psi \in \Psi_{LA}$ such that $\norm{\phi - \psi}_{U}<\epsilon$. Here, $\Psi_{LA}$ is the set of LA-VPNets.
\end{theorem}

\subsection{Proofs}
To complete the proofs, we first demonstrate the theory of LA-VPNet and start with the following auxiliary lemma. In this section, in order to simplify the subscript, we denote
\begin{equation*}
L^i = \left\{ S \in \R^{D \times D}\ \Bigg|\ S = \begin{pmatrix} I_{(i-1) \times (i-1)} & 0_{(i-1) \times 1} & 0_{(i-1) \times (D-i)} \\ U_{1 \times (i-1)} & 1 & V_{1 \times (D-i)} \\ 0_{(D-i) \times (i-1)} & 0_{(D-i) \times 1} & I_{(D-i) \times (D-i)} \end{pmatrix}\ \right\}
\end{equation*}
for $i=1,\cdots, D$ and let
\begin{equation*}
L^{i,i+1} = \left\{ P \in \R^{D \times D}\ \Bigg|\ \det P =1,\ P = \begin{pmatrix} I_{(i-1) \times (i-1)} & &0_{(i-1) \times (D-i+1)} \\ \ & U_{2 \times D} & \\0_{(D-i-1) \times (i+1)} & \ & I_{(D-i-1) \times (D-i-1)} \end{pmatrix}\ \right\}
\end{equation*}
with $i=1,\cdots, D-1$. Here, the subscript indicates the shape of matrices and $I_{s \times s}$ is the $s$-by-$s$ identity matrix.

\begin{lemma}\label{lem:comerror}
We assume that $\varphi^1, \cdots,\varphi^N$ are some functions from $\R^D$ to $\R^D$, and that $\varphi_k$ is Lipschitz on any compact set for $1\leq k \leq N$. If $\varphi^k \in \overline{\Psi}^{U}$ holds on any compact $U$ for $1\leq k \leq N$, then $\varphi^N\circ \cdots \circ \varphi^1 \in \overline{\Psi}^U$ holds on any compact $U$. Here, $\overline{\Psi}^U$ denotes the closure of $\Psi$ in $C(U)$ where $\Psi = \Psi_{LA}$ or $\Psi = \Psi_{R}$.
\end{lemma}
\begin{proof}
The proof is a direct extension of the proof of lemma 1 in \cite{zhu2022approximation}.
\end{proof}
Next we study the approximation of the linear modules, and aim to show that any matrix with determinant of $1$ can be decomposed into the product of the elements in $\cup_{i=1}^{D} \cup_{j=i+1}^{D+1} L^{i:j}$.
\begin{lemma}\label{lem:facA}
For $2\leq d \leq D$, given matrix $A^1 \in \R^{d\times d}$ and $A^2 \in \R^{d\times (D-d)}$, if the determinant of $A^1$ is $1$, then for any $\varepsilon>0$, there exist $P^{i,i+1} \in L^{i,i+1}$ for $i = 1,\cdots, d-1$ such that
\begin{equation*}
\norm{\begin{pmatrix}  A^1 & A^2 \\0_{(D-d) \times d} &I_{(D-d) \times (D-d)}\end{pmatrix}-P^{1,2}P^{2,3}\cdots P^{d-2,d-1}P^{d-1,d}}<\varepsilon .
\end{equation*}
\end{lemma}
\begin{proof}
We prove this by induction on $2\leq d \leq D$. To begin with, the case when $d=2$ is obvious. In addition, suppose now that the conclusion holds for $d-1$. For any $0<\varepsilon<1$, let
\begin{equation*}
p=(A^1[d,\ :d],A^1[d,d]+\delta,A^2[d,\ :\ ]) \in \R^{1\times D},
\end{equation*}
and we can choose $\delta$ and row vector $q$ such that
\begin{equation*}
P^{d-1,d} = \begin{pmatrix} I_{(d-2) \times (d-2)} & &0_{(d-2) \times (D-d+2)} \\ \ & q & \\\ & p & \\0_{(D-d) \times d} & \ & I_{(D-d) \times (D-d)} \end{pmatrix} \in L^{d-1,d}.
\end{equation*}
Add $\delta$ to $A^1[d,d]$ and denote the new matrix as $\tilde{A}^1$, then it is easy to check that there exist $B^1 \in \R^{(d-1)\times (d-1)}$ and $B^2 \in \R^{(d-1)\times (D-d+1)}$, such that
\begin{equation*}
\begin{pmatrix}  \tilde{A}^1 & A^2 \\0_{(D-d) \times d} &I_{(D-d) \times (D-d)}\end{pmatrix} = \begin{pmatrix}  B^1 & B^2 \\0_{(D-d+1) \times (d-1)} &I_{(D-d+1) \times (D-d+1)}\end{pmatrix}P^{d-1,d}.
\end{equation*}
By induction, for any $\tilde{\varepsilon}>0$, there exist $P^{i,i+1} \in L^{i,i+1}$ for $i = 1,\cdots, d-2$ such that
\begin{equation*}
\norm{\begin{pmatrix}  B^1 & B^2 \\0_{(D-d+1) \times (d-1)} &I_{(D-d+1) \times (D-d+1)}\end{pmatrix}-P^{1,2}P^{2,3}\cdots P^{d-2,d-1}}<\tilde{\varepsilon} .
\end{equation*}
Consequently,
\begin{equation*}
\norm{\begin{pmatrix}  \tilde{A}^1 & A^2 \\0_{(D-d) \times d} &I_{(D-d) \times (D-d)}\end{pmatrix} -P^{1,2}P^{2,3}\cdots P^{d-2,d-1}P^{d-1,d}}<\norm{P^{d-1,d}}\tilde{\varepsilon}.
\end{equation*}

If $A^1[d,d] \neq 0$, taking $\delta=0$ results in $\norm{P^{d-1,d}}$ being a constant depending on $A_1,A_2$. Thus taking $\tilde{\varepsilon} = \varepsilon$ concludes the induction. Otherwise, from the definition of $P^{d-1,d}$, we derive that $\norm{P^{d-1,d}}< C_1+\frac{C_2}{\delta}$, where $C_1,C_2$ are constants depending on $A_1,A_2$. Taking $\tilde{\varepsilon} = \varepsilon \cdot \delta$ and $\delta = \varepsilon$ completes the induction and hence the proof.
\end{proof}

\begin{lemma}\label{lem:facB}
For $i=1,\cdots, D-1$, given $P^{i,i+1} \in \R^{D\times D} \in L^{i,i+1}$, there exist $S^{i+1}_{1}, S^{i+1}_{2}, S^{i}_{1}, S^{i}_{2}$ and $S^i_m \in L^i$ for $m=1,2$ as well as $T^{i+1}_{1}, T^{i+1}_{2}, T^{i}_{1}, T^{i}_{2}$ and $T^i_m \in L^i$ for $m=1,2$ such that
\begin{equation*}
P^{i,i+1} = S^{i+1}_{1} S^{i}_{1} S^{i+1}_{2} S^{i}_{2} = T^{i}_{1} T^{i+1}_{1} T^{i}_{2} T^{i+1}_{2}.
\end{equation*}
\end{lemma}
\begin{proof}
We omit the zero elements and rewrite $P^{i,i+1}$ as
\begin{equation*}
P^{i,i+1} = \begin{pmatrix} I_{(i-1) \times (i-1)} & \ & \ \\ \begin{pmatrix}U_{11} \\ U_{21} \end{pmatrix} & U_0 & \begin{pmatrix}U_{12} \\ U_{22} \end{pmatrix} \\\ \ & \ & I_{(D-i-1) \times (D-i-1)} \end{pmatrix},
\end{equation*}
where $U_0 \in \R^{2 \times 2},\ U_{11}, U_{21} \in \R^{1 \times (i-1)}, \ U_{12}, U_{22} \in \R^{1 \times (D-i-1)}$. The fact that the determinant of $P^{i,i+1}$ equals to $1$ implies that the determinant of $U_0$ is 1. Thus, this results in $U_0$ being a symplectic matrix since $U_0 \in \R^{2 \times 2}$. By \cite{jin2020unit}, there exist $a,b,c,d \in \R$ such that
\begin{equation*}
U_0=\begin{pmatrix} 1 & 0\\a &1 \end{pmatrix}\begin{pmatrix} 1 & b\\0 &1 \end{pmatrix}\begin{pmatrix} 1 & 0\\c &1 \end{pmatrix}\begin{pmatrix} 1 & d\\0 &1 \end{pmatrix}.
\end{equation*}
Taking
\begin{equation*}
\begin{aligned}
&S^{i}_{1} = \begin{pmatrix} I_{(i-1) \times (i-1)} & \ & \ \\ \begin{pmatrix}U_{11} \\0 \end{pmatrix} & \begin{pmatrix} 1 & b\\0 &1 \end{pmatrix} & \begin{pmatrix} U_{12} \\ 0 \end{pmatrix} \\\ \ & \ & I_{(D-i-1) \times (D-i-1)} \end{pmatrix},\quad S^{i}_{2} = \begin{pmatrix} I_{(i-1) \times (i-1)} & \ & \ \\ 0 & \begin{pmatrix} 1 & d\\0 &1 \end{pmatrix} & 0 \\\ \ & \ & I_{(D-i-1) \times (D-i-1)} \end{pmatrix}
\end{aligned}
\end{equation*}
and
\begin{equation*}
\begin{aligned}
&S^{i+1}_{1} = \begin{pmatrix} I_{(i-1) \times (i-1)} & \ & \ \\ \begin{pmatrix} 0 \\ U_{21}-a U_{11} \end{pmatrix} & \begin{pmatrix} 1 & 0\\a &1 \end{pmatrix} & \begin{pmatrix} 0\\ U_{22}-a U_{12} \end{pmatrix} \\\ \ & \ & I_{(D-i-1) \times (D-i-1)} \end{pmatrix},\quad S^{i+1}_{2} = \begin{pmatrix} I_{(i-1) \times (i-1)} & \ & \ \\ 0 & \begin{pmatrix} 1 & 0\\c &1 \end{pmatrix} & 0 \\\ \ & \ & I_{(D-i-1) \times (D-i-1)}, \end{pmatrix},
\end{aligned}
\end{equation*}
we can readily check that
\begin{equation*}P^{i,i+1} = S^{i+1}_{1} S^{i}_{1}S^{i+1}_{2} S^{i}_{2}.
\end{equation*}
Finally, expressing $(P^{i,i+1})^{-1}$ in the above approach implies
\begin{equation*}
P^{i,i+1} = T^{i}_{1} T^{i+1}_{1} T^{i}_{2} T^{i+1}_{2},
\end{equation*}
where $T^i_m \in L^i$ for $m=1,2$.
\end{proof}
From Lemma \ref{lem:facA} and \ref{lem:facB}, we know that the proposed linear module $\mathcal{M}_L$ can approximate all the linear volume-preserving maps. With this result, we are able to present the following lemma.
\begin{lemma}\label{lem:la}
Given compact $U \subset \R^D$, we have $\Psi_{R} \subset \overline{\Psi_{LA}}^U$, where $\overline{\Psi_{LA}}^U$ denotes the closure of $\Psi_{LA}$ in $C(U)$.
\end{lemma}
\begin{proof}
We consider a residual module as
\begin{equation*}
u(x)= \begin{pmatrix}x[\ :i]\\ x[i:j] + a \sigma(K\overline{x[i:j]}+b)\\ x[j:\ ] \end{pmatrix},
\end{equation*}
where $K \in \R^{M (D-j+i)\times (D-j+i)}, b\in \R^{M (D-j+i)}, a\in \R^{(j-i) \times M (D-j+i)}$ and $K = (K_1^T, \cdots, K_M^T)^T$ with $ K_i \in \R^{(D-j+i) \times (D-j+i)} $ and $\det K_i \neq 0$ for $i=1, \cdots, M$.

Denote $a= (a_1^T, \cdots, a_M^T)^T,\ a_i \in \R^{(j-i) \times (D-j+i)}$ and $b= (b_1^T, \cdots, b_M^T)^T, \ b_i \in \R^{D-j+i}$. We take
\begin{equation*}
\begin{aligned}
&L_m=\begin{pmatrix} K_m[1:i,1:i]& 0_{i-1,j-i} & K_m[1:i,i:\ ]\\0_{j-i,i-1} & K_m' & 0_{j-i,D+1-j} \\ K_m[i:\ ,1:i] & 0_{D-i,j-i}& K_m[i:\ ,i:\ ]\end{pmatrix}, \quad S_m=\begin{pmatrix} K_m^{-1}[1:i,1:i]& 0_{i-1,j-i} & K_m^{-1}[1:i,i:\ ]\\0_{j-i,i-1} & K_m'^{-1} & 0_{j-i,D+1-j} \\ K_m^{-1}[i:\ ,1:i] & 0_{D-i,j-i}& K_m^{-1}[i:\ ,i:\ ]\end{pmatrix} \\
\end{aligned}
\end{equation*}
and
\begin{equation*}
A_m(x)= \begin{pmatrix}x[\ :i]\\ x[i:j] + K_m' a_m \sigma(\overline{x[i:j]})\\ x[j:\ ] \end{pmatrix}
\end{equation*}
where $K_m'\in \R^{(j-i) \times (j-i)}$ satisfies $\det K_m' \det K_m =1$ for $m=1, \cdots M$. For $m=1, \cdots, M$, we define
\begin{equation*}
\begin{aligned}
h_m &=S_m A_m\left(L_m x+\begin{pmatrix}b_m[\ :i]\\ 0_{j-i}\\b_m[i:\ ]\end{pmatrix}\right) - \begin{pmatrix}(K_m^{-1}b_m)[\ :i]\\ 0_{j-i}\\(K_m^{-1}b_m)[i:\ ]\end{pmatrix} =\begin{pmatrix} x[\ :i] \\x[i:j] + a\sigma (K_m \overline{x[i:j]} +b_m)\\ x[j:\ ]\end{pmatrix}.
\end{aligned}
\end{equation*}
Lemma \ref{lem:facA} and \ref{lem:facB} imply $h_m \in \overline{\Psi_{LA}}$. Subsequently, we can check that
\begin{equation*}
u(x) = h_M \circ \cdots \circ h_1,
\end{equation*}
and thus according to Lemma \ref{lem:comerror}, we know that $u \in \overline{\Psi_{LA}}^U$. For general residual modules, we can extend $K, a, b$ with some zero rows to meet the requirement of the width. Again, Lemma \ref{lem:comerror}, together with the fact that non-singular matrix is dense in the matrix set, we conclude the proof.
\end{proof}

Lemma \ref{lem:la} indicates that LA-VPNets do not change the approximation properties of R-VPNets and thus it is sufficient to prove Theorem \ref{thm:R}. Therefore, according to the lemma 5 in \cite{zhu2022approximation}, it remains to bridge the gap between smooth flow and $C^1$ flow to finish the main results. Next we state the well-known Gr{\"o}nwall's Inequality \cite{gronwall1919note}.
\begin{proposition}[Gr{\"o}nwall's Inequality]
Let $F:\R \rightarrow \R$ be a scalar function such that $F(t)\geq 0$ and $F(T) \leq AT + B \int_0^T F(t) dt$ with $A, B >0$. Then, $F(T) \leq A(e^{BT}-1)/B$.
\end{proposition}
Now, we are able to present the proofs of the main theorems.
\begin{proof}[Proof of Theorem \ref{thm:R}]
Given an compact set $U\subset \R^{D}$ and a volume-preserving flow map $\varphi_{\tau, T, f} \in \mathcal{VF}(U)$ with vector field $f$, the universal approximation theorem of neural networks with one hidden layer and sigmoid activation \cite{cybenko1989approximation,hornik1990universal} implies that for any $\varepsilon>0$, there exists a smooth neural networks $u$ such that
\begin{equation*}
\norm{u-f}_V \leq \varepsilon,
\end{equation*}
where
\begin{equation*}
V = \{\varphi_{\tau,t, f}(x)\ |\ x\in U,\ \tau\leq t\leq T\}.
\end{equation*}
Therefore, for any $x \in U$,
\begin{equation*}
\begin{aligned}
\norm{\varphi_{\tau, T, f}(x) -\varphi_{\tau, T, u}(x)} \leq & \norm{\int_{0}^T f\big(\varphi_{\tau, t, f}(x)\big) - u\big(\varphi_{\tau, t, u}(x) \big) dt}\\
\leq &\int_{0}^T \norm{f\big(\varphi_{\tau, t, f}(x)\big) - u\big(\varphi_{\tau, t, f}(x)\big)} dt + \int_{0}^T \norm{u\big(\varphi_{\tau, t, f}(x)\big) - u\big(\varphi_{\tau, t, u}(x)\big)} dt\\
\leq & \ T \cdot \varepsilon + L \int_{0}^T \norm{\varphi_{\tau, t, f}(x) - \varphi_{\tau, t, u}(x)} dt.
\end{aligned}
\end{equation*}
where $L$ is the Lipschitz constant of $U$. Applying Gr{\"o}nwall's inequality, we obtain that
\begin{equation*}
\norm{\varphi_{\tau, T, f}(x) -\varphi_{\tau, T, u}(x)} \leq \varepsilon \cdot (e^{ T L}-1)/L.
\end{equation*}
Since $u$ is smooth, according to the lemma 6 in \cite{zhu2022approximation}, there exists a VPNet $\hat{\psi}_{net}$ composed by basic modules
\begin{equation*}\label{eq:up}
\begin{aligned}
&\hat{x}[\ :s] = x[\ :s] + f_{net}(x[s:\ ]),\\ &\hat{x}[s:\ ] = x[s:\ ],
\end{aligned}
\quad \text{and}\quad
\begin{aligned}
&\hat{x}[\ :s ] = x[\ :s ],\\
&\hat{x}[s:\ ] = x[s:\ ] + f_{net}(x[\ :s ]),\\
\end{aligned}
\end{equation*}
such that $\norm{\varphi_{\tau, T, u}- \hat{\psi}_{net}}_U \leq \varepsilon/2$. Clearly, the above modules can be approximated by the composition of several residual modules defined in this paper. This fact together with Lemma \ref{lem:comerror} yields that there exists a R-VPNet $\psi_{net}$ such that
\begin{equation*}
\norm{\varphi_{\tau, T, u}- \psi_{net}}_U \leq \norm{\varphi_{\tau, T, u}- \hat{\psi}_{net}}_U + \norm{\hat{\psi}_{net}- \psi_{net}}_U \leq \varepsilon.
\end{equation*}
Finally, we conclude that
\begin{equation*}
\norm{\varphi_{\tau, T, f}- \psi_{net}}_U \leq \big( (e^{ T L}-1)/L+1\big) \cdot \varepsilon,
\end{equation*}
which completes the proof.
\end{proof}
\begin{proof}[Proof of Theorem \ref{thm:LA}]
Combining Lemma \ref{lem:la} and Theorem \ref{thm:R}, we conclude the proof.
\end{proof}

\section{Numerical results}\label{sec: Numerical results}
In this section, we show the results of the proposed VPNets on two benchmark problems. Since volume-preserving is equivalent to symplecticity-preserving in 2-dimensional systems, we investigate the learning models in higher dimensions. The code accompanying this paper are publicly available at \url{https://github.com/Aiqing-Zhu/VPNets}.

\subsection{Experiment setting}
We summarize the overall setting of all experiments in this subsection. The experiments are performed in the Python 3.6 environment. We utilize the PyTorch library for neural network implementation. Here, 5 independent experiments are simulated for both cases, and we show the results with the lowest training loss. All of the R-VPNets used in the examples are of the form
\begin{equation*}
\Psi_R =(\mathcal{R}^{D:D+1} \circ \mathcal{R}^{1:2} \circ \mathcal{R}^{D:D+1})  \circ (\mathcal{R}^{D-1:D} \circ \mathcal{R}^{D:D+1} \circ \mathcal{R}^{D-1:D}) \circ \cdots \circ (\mathcal{R}^{1:2} \circ \mathcal{R}^{2:3} \circ \mathcal{R}^{1:2})
\end{equation*}
where $D$ is the dimension of the problem and $\mathcal{R}^{i:j}$ is defined as in Section \ref{sec: The VPNets architecture} with width $w =64$. The LA-VPNets are given as
\begin{equation*}
\begin{aligned}
\Psi_{LA} =& \mathcal{L}\circ (\mathcal{A}^{D:D+1}\circ \mathcal{L} \circ \mathcal{A}^{1:2}\circ \mathcal{L} \circ \mathcal{A}^{D:D+1}\circ \mathcal{L})\\
&\circ (\mathcal{A}^{D-1:D}\circ \mathcal{L} \circ \mathcal{A}^{D:D+1}\circ \mathcal{L} \circ \mathcal{A}^{D-1:D}\circ \mathcal{L}) \circ \cdots \circ (\mathcal{A}^{1:2}\circ \mathcal{L} \circ \mathcal{A}^{2:3}\circ \mathcal{L} \circ \mathcal{A}^{1:2}\circ \mathcal{L}),
\end{aligned}
\end{equation*}
where $\mathcal{L}$ is the linear module of the form
\begin{equation*}
\mathcal{L}(x) = \left(\prod_{i=1}^D S^{i:i+1}S^{i+1:i+2}S^{i:i+1} \right) x +b.
\end{equation*}
The activation function is chosen to be sigmoid for both VPNets. The trainable parameters in the VPNets are determined via minimizing MSE loss (\ref{eq:mse}) using the Adam algorithm \cite{kingma2014adam} from the PyTorch library for both examples. The learning rate is set to decay exponentially with linearly decreasing powers, i.e., the learning rate in the $n$-th epoch denoted as $lr_n$ is given by
\begin{equation*}
    lr_n =lr_0 * d^{-n/N},
\end{equation*}
where $N$ is the total epochs, $lr_0$ is the initial learning rate and $d$ is the decay coefficient. The training parameters for each examples are summarized in Table \ref{tab:training parameters and training loss}. For convenience, we also report the corresponding training loss in Table \ref{tab:training parameters and training loss}.

\begin{table}[htbp]
    \centering
    \begin{tabular}{p{90pt}p{10pt}p{80pt}p{60pt}p{10pt}p{80pt}p{60pt}}
        \toprule
        Problem & & \multicolumn{2}{c}{Volterra equations} & & \multicolumn{2}{c}{Charged particle dynamics} \\
        \cmidrule(lr){1-1} \cmidrule(lr){3-4} \cmidrule(lr){6-7}
        Network type & & R-VPNet  &LA-VPNet & & R-VPNet  &LA-VPNet\\
        \cmidrule(lr){1-1}  \cmidrule(lr){3-4} \cmidrule(lr){6-7}
        Parameters& & 2.3K & 0.2K & & 3.8K &0.6K\\
        Initial learning rate & & 0.01 & 0.01 & & 0.001 & 0.01 \\
       Decay coefficient & &1000 & 1000 & & 100 &100\\
        Epochs & & 300000 & 300000& & 500000 & 800000\\
        Training loss& & 3.82e-9 & 5.25e- 7 & & 1.75e-8 & 1.09e-7\\
        \bottomrule
    \end{tabular}
    \caption{Training parameters and training loss.}
    \label{tab:training parameters and training loss}
\end{table}
\subsection{Volterra equations}
Consider the three-dimensional Volterra equation:
\begin{equation*}
\begin{aligned}
\frac{d p}{d t} &= p\ (q-r),\\
\frac{d q}{d t} &= q\ (r-p),\\
\frac{d r}{d t} &= r\ (p-q).\\
\end{aligned}
\end{equation*}
Two trajectories with initial conditions $y_0 =(5,4.1,5.9), (5,3.9,6.1)$ and stepsize $h= 0.01$ are simulated and the first 75 points (about one period) are used as the training set.

To investigate the performance of the proposed models, we perform predictions starting from $y_0 = (5,4,6)$, $(5.2,4,5.8)$, $(4.9,4,6.1)$ using trained VPNets. The performance is shown in Fig. \ref{fig:lv}. Although the test trajectories are far away from the training data, the proposed VPNets capture the dynamic evolution of the system perfectly. These results demonstrate that the VPNets are able to record the fine structures in the learned discrete data, and the serving algorithm correctly predicts the volume-preserving dynamics.
\begin{figure}[htbp]
     \centering
     \includegraphics[width=1\textwidth]{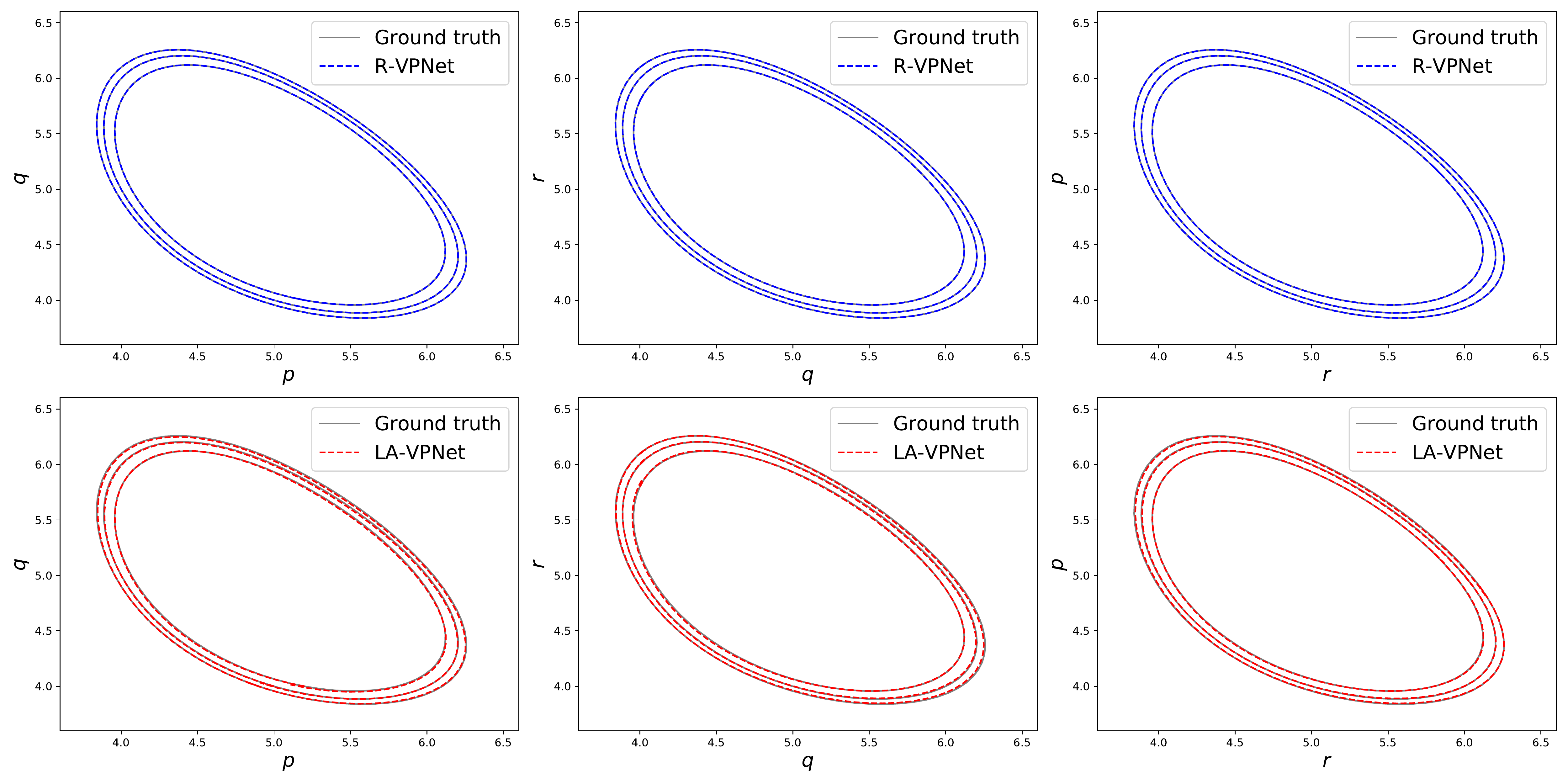}
     \caption{Results for Volterra equations of R-VPNets ({Top}) and LA-VPNets ({Bottom})}
     \label{fig:lv}
 \end{figure}

\subsection{Charged particle dynamics}
We next consider a single charged particle model with the Lorentz force described as
\begin{equation}\label{eq:charged particle}
\begin{aligned}
\frac{d x}{d t}&=v ,\\
\frac{d v}{d t} &=\frac{q}{m} (E(x)+v \times B(x)),
\end{aligned}
\end{equation}
where $m$ is the mass and $q$ is the electric charge, $\mathbf{x} = (x_1,x_2,x_3)\in \R^3$ and $\mathbf{v}=(v_1,v_2,v_3) \in \R^3$ represent the position and velocity of the charged particle, respectively. For simplicity, we set $m=q=1$. The dynamics is governed by the electric field $E(x)$ and the magnetic field $B(x)$, and in this section we consider a time-independent and non-uniform electromagnetic field
\begin{equation*}
B(x) = (0,0,R), \ E(x) = \frac{10^{-2}}{R^3}(x_1,x_2,0)
\end{equation*}
with $R = (x_1^2+x_2^2)^{1/2}$. The energy
\begin{equation}
H = \frac{1}{2} (v_1^2+v_2^2) + \frac{10^{-2}}{(x_1^2+x_2^2)^{1/2}}
\end{equation}
is an invariant which will be utilized for evaluating the performance of different models. Equation (\ref{eq:charged particle}) has all diagonal elements of $f'$ identically zero and thus is source-free. The exact solution is computed by Boris algorithm with very fine mesh. We refer to \cite{qin2013why,tu2016a} for more details and numerical algorithms about the charged particle model.

We aim to learn a single trajectory starting from $\mathbf{x}_0=[0.1,1,0]$, $\mathbf{v}_0=[1,0.2,0]$ which is a 4-dimensional dynamics since $x_3\equiv v_3 \equiv 0$. For the trajectory, 100 pairs of snapshots at $(ih, (i+1)h), i=0,\cdots, 99$ with shared data step $h=0.5$ are used as the training data set. We remark that due to the residual connection, R-VPNets can circumvent degradation and thus are easy to optimize. Therefore, we increase the training
epochs of the LA-VPNet here for fair comparison.

After training, we used the trained model to compute the flow starting at $t=50$. Figure \ref{Results for charged particle systems} shows the prediction results of different models from $t=50$ to $t=125$. We also report the conservation of energy and the global error to demonstrate the performance. The learned system accurately reproduces the phase portrait and preserves the energy error within a reasonable range. Although LA-VPNet has a bigger training loss, its global error is slightly smaller than that of R-VPNets.
 \begin{figure}[htbp]
     \centering
     \includegraphics[width=1\textwidth]{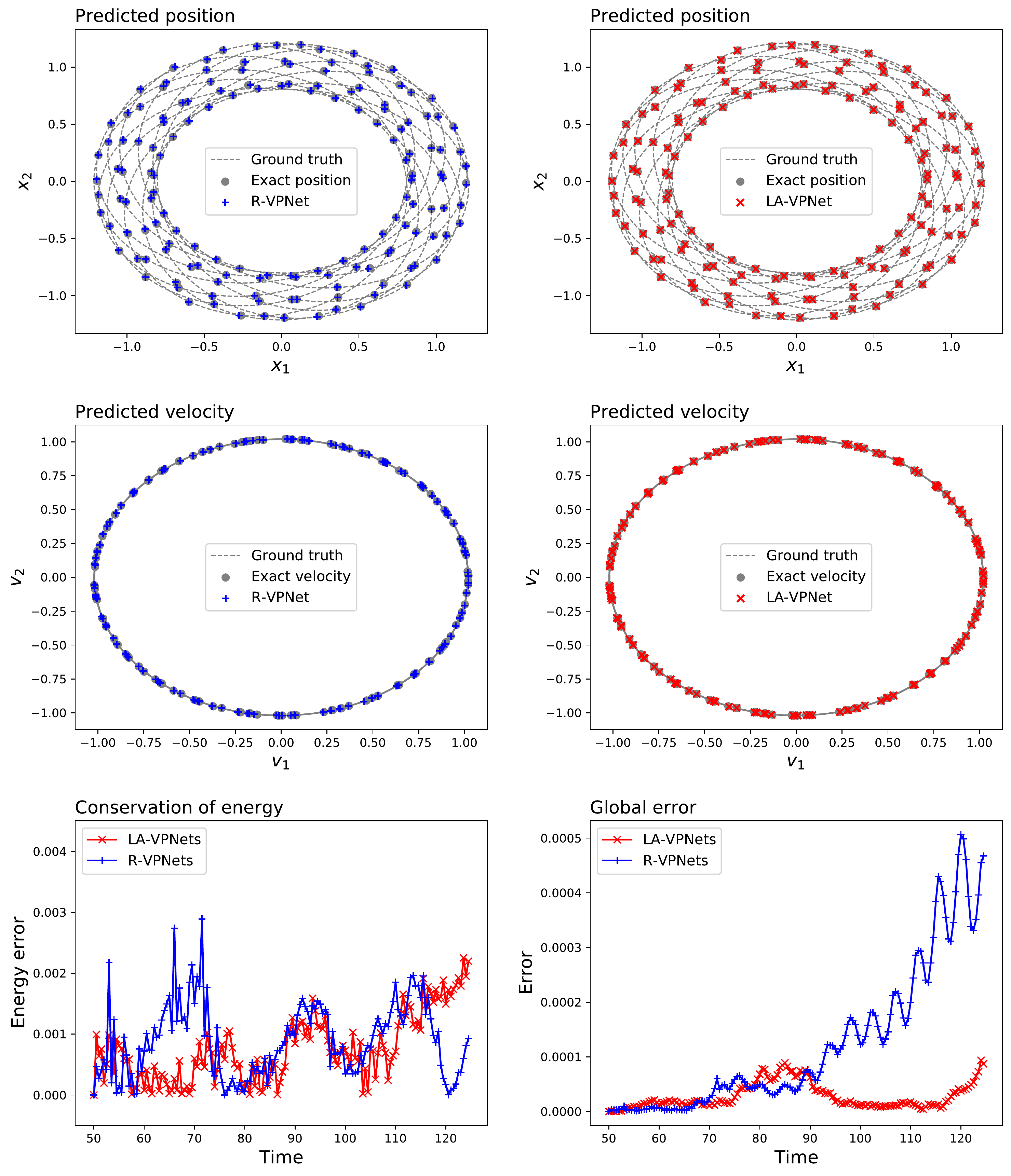}
     \caption{{Results for charged particle systems.} The top row displays the $(x_1,x_2)$orbit predicted by two models, and the middle row displays the velocity predicted by two models. Both VPNets reproduce the dynamics accurately. The absolute energy error for the predicted dynamics and global error of two models are presented in the bottom row. Both VPNets can limit the energy error within a reasonable range.
     }
     \label{Results for charged particle systems}
 \end{figure}

\section{Conclusions and future works}\label{sec: Conclusions and future works}
This work focuses on embedding prior geometric structures, i.e., volume-preserving, into the training data. The main contribution of this work is to propose network models that are intrinsic volume-preserving to identify source-free dynamics. In addition, we prove the approximation theory to show that our models are able to approximate any volume-preserving flow map. Numerical experiments verify our theoretical results and demonstrate the validity of the proposed VPNets in terms of generalization and prediction.

One limitation of our work is the long-time prediction. The approximation theorem only characterizes the local error while the long-time analysis remains open. It will be interesting to improve the long-time prediction behaviors and build corresponding error estimations like symplecticity-preserving networks \cite{chen2021data}.

Our approach is one method of constructing volume-preserving networks. We would also like to explore other approaches, such as generating functions and continuous models, to develop volume-preserving models.

\section*{Acknowledgments}
This research is supported by the Major Project on New Generation of Artificial Intelligence from MOST of China (Grant No. 2018AAA0101002), National Natural Science Foundation of China (Grant Nos. 11775222, 11901564 and 12171466), and the Geo-Algorithmic Plasma Simulator (GAPS) Project.

\bibliography{ref}

\begin{thebibliography}{10}
\expandafter\ifx\csname url\endcsname\relax
  \def\url#1{\texttt{#1}}\fi
\expandafter\ifx\csname urlprefix\endcsname\relax\def\urlprefix{URL }\fi
\expandafter\ifx\csname href\endcsname\relax
  \def\href#1#2{#2} \def\path#1{#1}\fi

\bibitem{bongard2007automated}
J.~Bongard, H.~Lipson, Automated reverse engineering of nonlinear dynamical
  systems, Proceedings of the National Academy of Sciences 104~(24) (2007)
  9943--9948.

\bibitem{schmidt2009distilling}
M.~Schmidt, H.~Lipson, Distilling free-form natural laws from experimental
  data, science 324~(5923) (2009) 81--85.

\bibitem{raissi2017machine}
M.~Raissi, P.~Perdikaris, G.~E. Karniadakis, Machine learning of linear
  differential equations using gaussian processes, Journal of Computational
  Physics 348 (2017) 683--693.

\bibitem{kocijan2005dynamic}
J.~Kocijan, A.~Girard, B.~Banko, R.~Murray-Smith, Dynamic systems
  identification with gaussian processes, Mathematical and Computer Modelling
  of Dynamical Systems 11~(4) (2005) 411--424.

\bibitem{brunton2017chaos}
S.~L. Brunton, B.~W. Brunton, J.~L. Proctor, E.~Kaiser, J.~N. Kutz, Chaos as an
  intermittently forced linear system, Nature communications 8~(1) (2017) 1--9.

\bibitem{anderson1996comparison}
J.~Anderson, I.~Kevrekidis, R.~Rico-Martinez, A comparison of recurrent
  training algorithms for time series analysis and system identification,
  Computers \& chemical engineering 20 (1996) S751--S756.

\bibitem{gonzalez1998identification}
R.~Gonz{\'a}lez-Garc{\'\i}a, R.~Rico-Mart{\`\i}nez, I.~G. Kevrekidis,
  Identification of distributed parameter systems: A neural net based approach,
  Computers \& chemical engineering 22 (1998) S965--S968.

\bibitem{rico1994continuous}
R.~Rico-Martinez, J.~Anderson, I.~Kevrekidis, Continuous-time nonlinear signal
  processing: a neural network based approach for gray box identification, in:
  Proceedings of IEEE Workshop on Neural Networks for Signal Processing, IEEE,
  1994, pp. 596--605.

\bibitem{rico1993continuous}
R.~Rico-Martinez, I.~G. Kevrekidis, Continuous time modeling of nonlinear
  systems: A neural network-based approach, in: IEEE International Conference
  on Neural Networks, IEEE, 1993, pp. 1522--1525.

\bibitem{chen2018neural}
T.~Q. Chen, Y.~Rubanova, J.~Bettencourt, D.~Duvenaud, Neural ordinary
  differential equations, in: Advances in Neural Information Processing Systems
  31, 2018, pp. 6572--6583.

\bibitem{kolter2019learning}
J.~Z. Kolter, G.~Manek, Learning stable deep dynamics models, in: Advances in
  Neural Information Processing Systems 32, 2019, pp. 11126--11134.

\bibitem{qin2019data}
T.~Qin, K.~Wu, D.~Xiu, Data driven governing equations approximation using deep
  neural networks, Journal of Computational Physics 395 (2019) 620--635.

\bibitem{raissi2018multistep}
M.~Raissi, P.~Perdikaris, G.~E. Karniadakis, Multistep neural networks for
  data-driven discovery of nonlinear dynamical systems, arXiv preprint
  arXiv:1801.01236.

\bibitem{yu2020onsagernet}
H.~Yu, X.~Tian, W.~E, Q.~Li, Onsager{N}et: Learning stable and interpretable
  dynamics using a generalized onsager principle, arXiv preprint
  arXiv:2009.02327.

\bibitem{zhang2021gfinns}
Z.~Zhang, Y.~Shin, G.~E. Karniadakis, Gfinns: Generic formalism informed neural
  networks for deterministic and stochastic dynamical systems, arXiv preprint
  arXiv:2109.00092.

\bibitem{lu2021learning}
L.~Lu, P.~Jin, G.~Pang, Z.~Zhang, G.~E. Karniadakis, Learning nonlinear
  operators via deeponet based on the universal approximation theorem of
  operators, Nature Machine Intelligence 3~(3) (2021) 218--229.

\bibitem{celledoni2021structure}
E.~Celledoni, M.~J. Ehrhardt, C.~Etmann, R.~I. McLachlan, B.~Owren, C.-B.
  SCHONLIEB, F.~Sherry, Structure-preserving deep learning, European Journal of
  Applied Mathematics 32~(5) (2021) 888--936.

\bibitem{bertalan2019learning}
T.~Bertalan, F.~Dietrich, I.~Mezi{\'c}, I.~G. Kevrekidis, On learning
  hamiltonian systems from data, Chaos: An Interdisciplinary Journal of
  Nonlinear Science 29~(12) (2019) 121107.

\bibitem{chen2020symplectic}
Z.~Chen, J.~Zhang, M.~Arjovsky, L.~Bottou, Symplectic recurrent neural
  networks, in: 8th International Conference on Learning Representations,
  {ICLR} 2020, OpenReview.net, 2020.

\bibitem{greydanus2019hamiltonian}
S.~Greydanus, M.~Dzamba, J.~Yosinski, {Hamiltonian} neural networks, in:
  Advances in Neural Information Processing Systems 32, 2019, pp. 15353--15363.

\bibitem{tong2021symplectic}
Y.~Tong, S.~Xiong, X.~He, G.~Pan, B.~Zhu, Symplectic neural networks in taylor
  series form for hamiltonian systems, Journal of Computational Physics 437
  (2021) 110325.

\bibitem{wu2020structure}
K.~Wu, T.~Qin, D.~Xiu, Structure-preserving method for reconstructing unknown
  hamiltonian systems from trajectory data, SIAM Journal on Scientific
  Computing 42~(6) (2020) A3704--A3729.

\bibitem{xiong2021nonseparable}
S.~Xiong, Y.~Tong, X.~He, S.~Yang, C.~Yang, B.~Zhu, Nonseparable symplectic
  neural networks, in: 9th International Conference on Learning
  Representations, {ICLR} 2021, OpenReview.net, 2021.

\bibitem{zhong2020symplectic}
Y.~D. Zhong, B.~Dey, A.~Chakraborty, Symplectic ode-net: Learning hamiltonian
  dynamics with control, in: 8th International Conference on Learning
  Representations, {ICLR} 2020, OpenReview.net, 2020.

\bibitem{du2021discovery}
Q.~Du, Y.~Gu, H.~Yang, C.~Zhou, The discovery of dynamics via linear multistep
  methods and deep learning: Error estimation, arXiv preprint arXiv:2103.11488.

\bibitem{keller2021discovery}
R.~T. Keller, Q.~Du, Discovery of dynamics using linear multistep methods, SIAM
  Journal on Numerical Analysis 59~(1) (2021) 429--455.

\bibitem{zhu2020inverse}
A.~Zhu, P.~Jin, Y.~Tang, Inverse modified differential equations for discovery
  of dynamics, arXiv preprint arXiv:2009.01058.

\bibitem{chen2021data}
R.~Chen, M.~Tao, Data-driven prediction of general hamiltonian dynamics via
  learning exactly-symplectic maps, in: M.~Meila, T.~Zhang (Eds.), Proceedings
  of the 38th International Conference on Machine Learning, {ICML} 2021, Vol.
  139, {PMLR}, 2021, pp. 1717--1727.

\bibitem{jin2020sympnets}
P.~Jin, Z.~Zhang, A.~Zhu, Y.~Tang, G.~E. Karniadakis, Sympnets: Intrinsic
  structure-preserving symplectic networks for identifying hamiltonian systems,
  Neural Networks 132 (2020) 166 -- 179.

\bibitem{hersch2008dynamical}
M.~Hersch, F.~Guenter, S.~Calinon, A.~Billard, Dynamical system modulation for
  robot learning via kinesthetic demonstrations, IEEE Transactions on Robotics
  24~(6) (2008) 1463--1467.

\bibitem{levinson2011towards}
J.~Levinson, J.~Askeland, J.~Becker, J.~Dolson, D.~Held, S.~Kammel, J.~Z.
  Kolter, D.~Langer, O.~Pink, V.~Pratt, et~al., Towards fully autonomous
  driving: Systems and algorithms, in: 2011 IEEE intelligent vehicles symposium
  (IV), IEEE, 2011, pp. 163--168.

\bibitem{lake2017building}
B.~M. Lake, T.~D. Ullman, J.~B. Tenenbaum, S.~J. Gershman, Building machines
  that learn and think like people, Behavioral and brain sciences 40.

\bibitem{marcus2020next}
G.~Marcus, The next decade in ai: four steps towards robust artificial
  intelligence, arXiv preprint arXiv:2002.06177.

\bibitem{hairer2006geometric}
E.~Hairer, C.~Lubich, G.~Wanner, Geometric numerical integration:
  structure-preserving algorithms for ordinary differential equations, Vol.~31,
  Springer Science \& Business Media, 2006.

\bibitem{macdonald2021volume}
G.~MacDonald, A.~Godbout, B.~Gillcash, S.~Cairns, Volume-preserving neural
  networks, arXiv preprint arXiv:1911.09576.

\bibitem{bajars2021locally}
J.~Baj{\=a}rs, Locally-symplectic neural networks for learning
  volume-preserving dynamics, arXiv preprint arXiv:2109.09151.

\bibitem{dinh2015nice}
L.~Dinh, D.~Krueger, Y.~Bengio, {NICE:} non-linear independent components
  estimation, in: 3rd International Conference on Learning Representations,
  {ICLR} 2015, 2015.

\bibitem{gomez2017the}
A.~N. Gomez, M.~Ren, R.~Urtasun, R.~B. Grosse, The reversible residual network:
  Backpropagation without storing activations, in: Advances in Neural
  Information Processing Systems 30, 2017, pp. 2214--2224.

\bibitem{zhu2022approximation}
A.~Zhu, P.~Jin, Y.~Tang, Approximation capabilities of measure-preserving
  neural networks, Neural Networks 147 (2022) 72--80.

\bibitem{jin2020unit}
P.~Jin, Y.~Tang, A.~Zhu, Unit triangular factorization of the matrix symplectic
  group, SIAM Journal on Matrix Analysis and Applications 41~(4) (2020)
  1630--1650.

\bibitem{gronwall1919note}
T.~H. Gronwall, Note on the derivatives with respect to a parameter of the
  solutions of a system of differential equations, Annals of Mathematics (1919)
  292--296.

\bibitem{cybenko1989approximation}
G.~Cybenko, Approximation by superpositions of a sigmoidal function,
  Mathematics of control, signals and systems 2~(4) (1989) 303--314.

\bibitem{hornik1990universal}
K.~Hornik, M.~Stinchcombe, H.~White, Universal approximation of an unknown
  mapping and its derivatives using multilayer feedforward networks, Neural
  Networks 3~(5) (1990) 551 -- 560.

\bibitem{kingma2014adam}
D.~P. Kingma, J.~Ba, Adam: {A} method for stochastic optimization, in: 3rd
  International Conference on Learning Representations, 2015.

\bibitem{qin2013why}
H.~Qin, S.~Zhang, J.~Xiao, J.~Liu, Y.~Sun, W.~M. Tang, Why is boris algorithm
  so good?, Physics of Plasmas 20~(8) (2013) 084503.

\bibitem{tu2016a}
X.~Tu, B.~Zhu, Y.~Tang, H.~Qin, J.~Liu, R.~Zhang, A family of new explicit,
  revertible, volume-preserving numerical schemes for the system of lorentz
  force, Physics of Plasmas 23~(12) (2016) 122514.

\end{thebibliography}

\end{document}